\documentclass[a4,a4wide]{article}
\usepackage{aaai}
\usepackage{times}
\usepackage{helvet}
\usepackage{courier}
\usepackage{amsthm,amssymb,array}
\usepackage{graphicx}
\usepackage{latexsym}
\usepackage[ruled,vlined,linesnumbered]{algorithm2e}
\usepackage{amsmath}
\usepackage[all]{xy}
\usepackage{multirow}
\usepackage{mathnotation}
\usepackage{url}
\usepackage{color}
\usepackage[nomargin,inline,index,draft]{fixme}
\fxusetheme{color}

\newcommand{\rl}{\text{r}}

\newcommand{\pr}{\Pi}

\newcommand{\dI}{\mathcal{I}}
\newcommand{\md}{\mathcal{D}}
\newcommand{\mb}{\mathcal{B}}
\newcommand{\Tb}{\mathit{BT}}
\newcommand{\Tc}{\mathit{CT}}
\newcommand{\tb}{\mathit{bt}}
\newcommand{\tc}{\mathit{ct}}
\newcommand{\Int}{\mathit{Int}}
\newcommand{\mI}{I}

\newcommand{\dpr}{\Delta}
\newcommand{\dpi}{\tuple{\dpr[\pr,\mb], P, N}}

\newcommand{\qry}{Q}
\newcommand{\ent}{E}

\newcommand{\mD}{{\bf{D}}}
\newcommand{\dx}{{\bf{D^{P}}}}
\newcommand{\dnx}{{\bf{D^{N}}}}
\newcommand{\dz}{{\bf{D^\emptyset}}}
\newcommand{\dxi}[1]{{\bf{D^{P}_{#1}}}}
\newcommand{\dnxi}[1]{{\bf{D^{N}_{#1}}}}
\newcommand{\dzi}[1]{{\bf{D^\emptyset_{#1}}}}

\newtheorem{definition}{Definition}

\newtheorem{property}{Property}
\newtheorem{proposition}{Proposition}

\newcommand{\nosemic}{\renewcommand{\@endalgocfline}{\relax}}
\newcommand{\dosemic}{\renewcommand{\@endalgocfline}{\algocf@endline}}

\begin{document}
\nocopyright
%
\title{Interactive Debugging of ASP Programs}
\author{Kostyantyn Shchekotykhin\\
University Klagenfurt, Austria\\
kostya@ifit.uni-klu.ac.at \\
}


\maketitle
\begin{abstract}
\begin{quote}

Broad application of answer set programming (ASP) for declarative problem solving requires the development of tools supporting the coding process. Program debugging is one of the crucial activities within this process. Modern ASP debugging approaches allow efficient computation of possible explanations of a fault. However, even for a small program a debugger might return a large number of possible explanations and selection of the correct one must be done manually. 
In this paper we present an interactive query-based ASP debugging method which extends previous approaches and finds a preferred explanation by means of observations. The system automatically generates a sequence of queries to a programmer asking whether a set of ground atoms must be true in all (cautiously) or some (bravely) answer sets of the program. Since some queries can be more informative than the others, we discuss query selection strategies which, given user's preferences for an explanation, can find the best query. That is, the query an answer of which reduces the overall number of queries required for the identification of a preferred explanation.
\end{quote}
\end{abstract}

\section{Introduction}

Answer set programming is a logic programming paradigm~\cite{Baral2003,Brewka2011,Gebser2012book} for declarative problem solving that has become popular during the last decades. The success of ASP is based on its fully declarative semantics~\cite{Gelfond1991} and  availability of efficient solvers, e.g.\ \cite{Simons2002,Leone2006,gekakaosscsc11a}. Despite a vast body of the theoretical research on foundations of ASP only recently the attention was drawn to the development of methods and tools supporting ASP programmers. The research in this direction focuses on a number of topics including integrated development environments~\cite{febbraro2011aspide,SeaLion2011,Sureshkumar2007}, visualization~\cite{Banda2008}, modeling techniques~\cite{Oetsch2011c} and, last but not least, debugging of ASP programs.

Modern ASP debugging approaches are mostly based on declarative strategies. The suggested methods use elegant techniques applying ASP itself to debug ASP programs. The idea is to transform a faulty program in  the special debugging program whose answer sets explain possible causes of a fault. These explanations are given by means of meta-atoms. 
A set of meta-atoms explaining a discrepancy between the set of actual and expected answer sets is called a \emph{diagnosis}. In practice considering all possible diagnoses might be inefficient. Therefore, modern debugging approaches apply built-in minimization techniques of ASP solvers to compute only diagnoses comprising the minimal number of elements. In addition, the number of diagnoses can be reduced by so called \emph{debugging queries}, i.e.\ sets of integrity constraints filtering out irrelevant diagnoses.

The computation of diagnoses is usually done by considering answer sets of a debugging program. In the approach of \cite{Syrjanen2006} a diagnosis corresponds to a set of meta-atoms indicating that a rule is removed from a program. 
\cite{Brain2007} use the tagging technique~\cite{delgrande2003} to obtain more fine-grained diagnoses. The approach differentiates between three types of problems: unsatisfied rules, unsupported atoms and unfounded loops. Each problem type is denoted by a special meta-predicate. Extraction of diagnoses can be done by a projection of an answer set of a debugging program to these meta-predicates. 
The most recent techniques~\cite{Gebser2008b,Oetsch2010a} apply \emph{meta-programming}, where a program over a meta language is used to manipulate a program over an object language. Answer sets of a debugging meta-program comprise sets of atoms over meta-predicates describing faults of the similar nature as in~\cite{Brain2007}. 

The main problem of the aforementioned declarative approaches is that in real-world scenarios it might be problematic for a programmer to provide a complete debugging query. Namely, in many cases a programmer can easily specify some small number of atoms that must be true in a desired answer set, but not a complete answer set. In this case the debugging system might return many alternative diagnoses. 
Our observations of the students developing ASP programs shows that quite often the programs are tested and debugged on some small test instances. This way of development is quite similar to modern programming  methodologies relying on unit tests~\cite{beck2003test} which were implemented in ASPIDE~\cite{Febbraro2013} recently. Each test case calls a program for a predefined input and verifies whether the actual output is the same as expected.
In terms of ASP, a programmer often knows a set of facts encoding the test problem instance and a set of output atoms encoding the expected solution of the instance. What is often unknown are the ``intermediate'' atoms used to derive the output atoms. However, because of these atoms multiple diagnoses are possible. The problem is to find and add these atoms to a debugging query in a most efficient way\footnote{A recent user study indicates that the same problem can be observed also in the area of ontology debugging (see https://code.google.com/p/rmbd/wiki/UserStudy for preliminary results).}.
Existing debugging systems~\cite{Brain2005,Gebser2008b,Oetsch2010a} can be used in an ``interactive'' mode in which a user specifies only a partial debugging query as an input. Given a set of diagnoses computed by a debugger the user extends the debugging query, thus, filtering out irrelevant answer sets of a meta-program. However, this sort of interactivity still requires a user to select and provide atoms of the debugging query manually.

Another diagnosis selection issue is due to inability of a programmer to foresee all consequences of a diagnosis, i.e.\ in some cases multiple interpretations might have the same explanation for not being answer sets. The simplest example is an integrity constraint which can be  violated by multiple interpretations. In this case the modification of a program accordingly to a selected diagnosis might have side-effects in terms of unwanted answer sets.
These two problem are addressed by our approach which helps a user to identify the \emph{target diagnosis}. The latter is the preferred explanation for a given set of atoms not being true in an answer set, on the one hand, and is not an explanation for unwanted interpretations, on the other. 

In this paper we present an \emph{interactive query-based debugging} method for ASP programs which differentiates between the diagnoses by means of additional observations~\cite{dekleer1987,jws12}. The latter are acquired by automatically generating a sequence of \emph{queries} to an oracle such as a user, a database, etc. Each answer 
is used to reduce the set of diagnoses until the target diagnosis is found. In order to construct queries our method uses the fact that in most of the cases different diagnoses explain why different sets of interpretations are not answer sets.
Consequently, we can differentiate between diagnoses by asking an oracle whether a set of atoms must be true or not in all/some interpretations relevant to the target diagnosis. Each set of atoms which can be used as a query is generated by the debugger automatically using discrepancies in the sets of interpretations associated with each diagnosis. Given a set of queries our method finds the best query according to a query selection strategy chosen by a user.

The suggested debugging approach can use a variety of query selection strategies. In this paper we discuss myopic and one step look-ahead strategies which are commonly used in active learning~\cite{SettlesBook}.
A myopic strategy implements a kind of greedy approach which in our case prefers queries that allow to reduce a set of diagnoses by half, regardless of an oracle's answer. 
The one step look-ahead strategy uses beliefs/preferences of a user for a cause/explanation of an error represented in terms of probability. Such a strategy selects those queries whose answers provide the most information gain, i.e.\ in whose answers a strategy is most uncertain about. New information provided by each answer is taken into account using Bayes-update. This allows the strategy to adapt its behavior on the fly.  

To the best of our knowledge there are no approaches to interactive query-based ASP debugging allowing automatic generation and selection of queries. 
The method presented in this paper suggests an extension of the current debugging techniques by an effective user involvement in the debugging process.
\section{Preliminaries}
\label{sect:prelim}
A \emph{disjunctive logic program} (DLP) $\pr$ is a finite set of rules of the form
\begin{align*}
h_1 \lor \dots \lor h_l \leftarrow b_1, \dots , b_m, not\ b_{m+1}, \dots, not\ b_n
\end{align*}
where all $h_i$ and $b_j$ are \emph{atoms} and $0 \leq l,m,n$. A \emph{literal} is an atom $b$ or its negation $not\ b$. Each \emph{atom} is an expression of the form $p(t_1, \dots, t_k)$, where $p$ is a predicate symbol and $t_1, \dots, t_k$ are terms. A \emph{term} is either a variable or a constant. The former is denoted by a string starting with an uppercase letter and the latter starting with a lowercase one.
A literal, a rule or a program is called \emph{ground}, if they are variable-free. A non-ground program $\pr$, its rules and literals can be grounded by substitution of variables with constants appearing in $\pr$. We denote the grounded instantiation of a program $\pr$ by $Gr(\pr)$ and by $At(\pr)$ the set of all ground atoms appearing in $Gr(\pr)$.

The set of atoms $H(r)=\setof{h_1,\dots,h_l}$ is called the \emph{head} of the rule $r$, whereas the set $B(r)=\setof{b_1, \dots , b_m, not\ b_{m+1}, \dots, not\ b_n}$ is the body of $r$. In addition, it is useful to differentiate between the sets $B^+(r)=\setof{b_1, \dots , b_m}$  and $B^-(r)=\setof{b_{m+1}, \dots, b_n}$ comprising \emph{positive} and \emph{negative} body atoms. A rule $c \in \pr$ with $H(c)=\emptyset$ is an \emph{integrity constraint} and a rule $f \in \pr$ with $B(f)=\emptyset$ is a \emph{fact}. A rule $r$ is \emph{normal}, if $|H(r)| \leq 1$. A \emph{normal program} includes only normal rules. 

An \emph{interpretation} $\mI$ for $\pr$ is a set of ground atoms $\mI \subseteq At(\pr)$. A rule $r\in Gr(\pr)$ is \emph{applicable} under $\mI$, if $B^+(r)\subseteq \mI$ and $B^-(r)\cap\mI=\emptyset$, otherwise the rule is \emph{blocked}. We say that $r$ is \emph{unsatisfied} by $\mI$, if it is applicable under $\mI$ and $H(r)\cap\mI = \emptyset$; otherwise $r$ is \emph{satisfied}. An interpretation $\mI$ is a \emph{model} of $\pr$, if it satisfies every rule $r\in Gr(\pr)$. For a ground program $Gr(\pr)$ and an interpretation $\mI$ the Gelfond-Lifschitz reduct is defined as $\pr^\mI=\setof{H(r)\leftarrow B^+(r) | r\in Gr(\pr), \mI \cap B^-(r)=\emptyset}$. $\mI$ is an \emph{answer set} of $\pr$, if $\mI$ is a minimal model of $\pr^\mI$~\cite{Gelfond1991}. The program $\pr$ is \emph{inconsistent}, if the set of all answer sets $AS(\pr)=\emptyset$.

\cite{Lee2005} provides another characterization of answer sets of a program $\pr$ based on the notion of support. Thus, a rule $r\in Gr(\pr)$ is a \emph{support}
for $A\subseteq At(\pr)$ with respect to $\mI$, if a rule $r$ is applicable under an interpretation $\mI$, $H(r)\cap A\neq\emptyset$ and $H(r)\cap\mI\subseteq A$. A support is \emph{external}, if $B^+(r) \cap A = \emptyset$. A set of ground atoms $A$ is \emph{unsupported} by $\pr$ with respect to $\mI$, if no rule in $Gr(\pr)$ supports it. A \emph{loop} is a non-empty set $L \subseteq At(\pr)$ such that for any two distinct atoms $a_i,a_j \in L$ there is a path $P$ in a positive dependency graph $G=(At(\pr),\setof{(h,b) | r \in Gr(\pr), h\in H(r),b\in B^+(r)})$, where $P\neq\emptyset$ and $P\subseteq L$. 
A loop $L$ is \emph{unfounded} by $\pr$ with respect to $\mI$, if no rule in $Gr(\pr)$ supports it externally; otherwise $L$ is founded. 
An interpretation $\mI$ is an answer set of $\pr$, iff $\mI$ is a model of $\pr$ such that each atom $a \in \mI$ is supported and each loop $L\subseteq \mI$ is founded~\cite{Lee2005}.

\section{Debugging of ASP programs} 
The approach presented in our paper is based on the meta-programming technique presented in~\cite{Gebser2008b}. This debugging method focuses on the identification of \emph{semantical errors} in a disjunctive logic program, i.e.\ disagreements between the actual answer sets of a program and the expected ones. 
The main idea is to use a program over a meta-language that manipulates another program over an object-language. The latter is a ground disjunctive program $\pr$ and the former is a non-ground normal logic program $\dpr[\pr]$. Each answer set of a meta-program $\dpr[\pr]$ comprises a set of atoms specifying an interpretation $\mI$ and a number of meta-atoms showing, why $\mI$ is not an answer set of a program $\pr$. In addition, the method guarantees that there is at least one answer set of $\dpr[\pr]$ for each interpretation $\mI$ which is not an answer set of $\pr$. 

The debugger provides explanations of four error types denoted by the corresponding \emph{error-indicating} predicates:
\begin{enumerate}
	\item \emph{Unsatisfied rules}: $\mI$ is not a classical model of $Gr(\pr)$ because the logical implication expressed by a rule $r$ is false under $\mI$. Atom $\mathit{unsatisfied}(id_r)$ in an answer set of $\dpr[\pr]$ expresses that a rule $r$ is unsatisfied by $\mI$, where $id_r$ is a unique identifier of a rule $r\in\pr$.
  \item \emph{Violated integrity constraints}: $\mI$ cannot be an answer set of $Gr(\pr)$, if a constraint $r$ is applicable under $\mI$. Atom $\mathit{violated}(id_r)$ indicates that $r$ is violated under $\mI$.
  \item \emph{Unsupported atoms}: there is no rule $r \in Gr(\pr)$ which allows derivation of $\setof{a} \subseteq \mI$ and, therefore, $\mI$ is not a minimal model of $\pr^{\mI}$. Each unsupported atom $a$ is indicated by an atom $\mathit{unsupported}(id_a)$ in an answer set of $\dpr[\pr]$, where $id_a$ is a unique identifier of an atom $a\in At(\pr)$.
  \item \emph{Unfounded loops}: $\mI$ is not a minimal model of $\pr^{\mI}$, if a loop $L \subseteq \mI$ is unfounded by a program $\pr$ with respect to $\mI$.  An atom $\mathit{ufLoop}(id_a)$ expresses that an atom $a \in At(\pr)$ belongs to the unfounded loop $L$.
\end{enumerate}
The set $Er(\dpr[\pr]) \subseteq At(\dpr[\pr])$ comprises all ground atoms over error-indicating predicates of the  meta-program $\dpr[\pr]$.

There are seven static modules in the meta-program $\dpr[\pr]$, see~\cite{Gebser2008b}. The input module $\pi_{in}$ comprises two sets of facts about atoms $\setof{atom(id_a)\leftarrow | a\in At(\pr)}$ and rules $\setof{rule(id_r)\leftarrow | r\in \pr}$ of the program $\pr$. Moreover, for each rule $r\in\pr$ the module $\pi_{in}$ defines which atoms are in $H(r)$, $B^+(r)$ and $B^-(r)$. Module $\pi_{int}$ generates an arbitrary  interpretation $\mI$ of a program $\pr$ as follows:
\begin{align*}
int(A) \leftarrow atom(A), not\ \overline{int}(A) \\
\overline{int}(A) \leftarrow atom(A), not\ int(A)
\end{align*}
where atom $int(A)$ is complimentary to the atom $\overline{int}(A)$, i.e.\ no answer set can comprise both atoms. 
The module $\pi_{ap}$ checks for every rule, whether it is applicable or blocked under $\mI$. The modules $\pi_{sat}$, $\pi_{supp}$ and $\pi_{ufloop}$ are responsible for the computation of at least one of the four explanations why $\mI$ is not an answer set of $\pr$ listed above. 
Note, $\pi_{ufloop}$ searches for unfounded loops only among atoms supported by $\pr$ with respect to $\mI$. This method ensures that each of the found loops is critical, i.e.\ it is a reason for $\mI$ not being an answer set of $\pr$. 
The last module, $\pi_{noas}$ restricts the answer sets of $\dpr[\pr]$ only to those that include any of the atoms over the error-indicating predicates.

The fault localization is done manually by means of \emph{debugging queries} which specify an interpretation to be investigated as a set of atoms, e.g.\ $I=\setof{a}$. Then $I$ is transformed into a finite set of constraints, e.g.\ $\setof{\leftarrow \overline{int}(id_a), \leftarrow int(id_b), \dots}$, pruning irrelevant answer sets of $\dpr[\pr]$.

\section{Fault localization in ASP programs}\label{sect:example}

In our work we extend the meta-programming approach by allowing a user to specify \emph{background theory} $\mb$ as well as \emph{positive} $P$ and \emph{negative} $N$ test cases. In this section we show how this additional information is used to keep the search focused only on relevant interpretations and diagnoses. 

Our idea of background knowledge is similar to~\cite{Brain2007} and suggests that some set of rules $\mb \subseteq \pr$ must be considered as correct by the debugger. 
In the meta-programming method the background theory can be accounted by addition of integrity constraints to $\pi_{noas}$ which prune all answer sets of $\dpr[\pr]$ suggesting that $r\in\mb$ is faulty.

\begin{definition} Let $\dpr[\pr]$ be a meta-program and $\mb \subseteq \pr$ a set of rules considered as correct. Then, a \emph{debugging program} $\dpr[\pr,\mb]$ is defined as an extension of $\dpr[\pr]$ with the rules: 
\begin{align*}
\{&\leftarrow \,  \mathit{rule}(id_r),\, \mathit{violated}(id_r), \\
  &\leftarrow \, \mathit{rule}(id_r),\ \mathit{unsatisfied}(id_r) \; |\; r \in \mb \}
\end{align*}
\end{definition}

In addition to background knowledge, further restrictions on the set of possible explanations of a fault can be made by means of test cases. 
\begin{definition} Let $\dpr[\pr,\mb]$ be a debugging program. 
A \emph{test case} for $\dpr[\pr,\mb]$ is a set $A \subseteq At(\dpr[\pr,\mb])$ of ground atoms over $int/1$ and $\overline{int}/1$ predicates. 
\end{definition}
The test cases are either specified by a user before a debugging session or acquired by a system automatically as we show in subsequent sections.

\begin{definition} Let $\dpr[\pr,\mb]$ be a debugging program and $\md \subseteq Er(\dpr[\pr,\mb])$ a set of atoms over error-indicating predicates. Then a \emph{diagnosis program} for $\md$ is defined as follows: \[\dpr[\pr,\mb,\md] :=\dpr[\pr,\mb] \cup \setof{\leftarrow d_i\;|\; d_i \in Er(\dpr[\pr,\mb]) \setminus \md}\] 
\end{definition}

In our approach we allow four types of test cases corresponding to the two ASP reasoning tasks~\cite{Leone2006}:
\begin{itemize}
  \item \emph{Cautious reasoning}: all atoms $a \in A$ are true in \emph{all} answer sets of the diagnosis program, resp.\ $\dpr[\pr,\mb,\md_t] \models_c A$, or not, resp.\ $\dpr[\pr,\mb,\md_t] \not\models_c A$. Cautiously true test cases are stored in the set $\Tc^+$ whereas cautiously false in the set $\Tc^-$.
	\item \emph{Brave reasoning}: all atoms $a\in A$ are true in \emph{some} answer set of the diagnosis program, resp.\ $\dpr[\pr,\mb,\md_t] \models_b A$, or not, resp.\ $\dpr[\pr,\mb,\md_t] \not\models_b A$. The set $\Tb^+$ comprises all bravely true test cases and the set $\Tb^-$ all bravely false test cases.
\end{itemize}

In the meta-programming approach we handle the test cases as follows:
Let $\dI$ be a set of ground atoms resulting from a projection of an answer set $as \in AS(\dpr[\pr,\mb,\md])$ to the predicates $int/1$ and $\overline{int}/1$. By $\Int(\dpr[\pr,\mb,\md])$ we denote a set comprising all sets $\dI_i$ for all $as_i \in AS(\dpr[\pr,\mb,\md])$. 
Each set of grounded atoms $\dI$ corresponds to an interpretation $\mI$ of the program $\pr$ which is not an answer set of $\pr$ as explained by $\md$. The set $\Int(\dpr[\pr,\mb,\md])$ comprises a meta representation of each such interpretation for a diagnosis $\md$. 
Given a set of grounded atoms $A$, we say that $\dI$ satisfies $A$ (denoted $\dI \models A$), if $A \subseteq \dI$. $\Int(\dpr[\pr,\mb,\md])$ satisfies $A$ (denoted $\Int(\dpr[\pr,\mb,\md]) \models A$), if $\dI \models A$ for every $\dI \in \Int(\dpr[\pr,\mb,\md])$. 
Analogously, we say that a set $\Int(\dpr[\pr,\mb,\md])$ is consistent with $A$, if there exists $\dI\in \Int(\dpr[\pr,\mb,\md])$ which satisfies $A$.

Let $A$ be a test case, then $\overline{A}$ denotes a complementary test case, i.e. $\overline{A}=\setof{\overline{int}(a)\; | \; int(a) \in A } \cup \setof{int(a)\; | \; \overline{int}(a) \in A }$. For the verification whether a diagnosis program $\dpr[\pr,\mb,\md]$ fulfills all test cases it is sufficient to check if the following conditions hold:
\begin{itemize}
	\item $\Int(\dpr[\pr,\mb,\md]) \models \tc^+ \quad \forall \tc^+ \in \Tc^+$
  \item $\Int(\dpr[\pr,\mb,\md]) \models \overline{\tb^-} \quad \forall \tb^- \in \Tb^-$
  \item $\Int(\dpr[\pr,\mb,\md]) \cup \overline{\tc^-} \; \text{is consistent} \;  \quad \forall \tc^- \in \Tc^-$
  \item $\Int(\dpr[\pr,\mb,\md]) \cup \tb^+ \; \text{is consistent} \; \quad \forall \tb^+ \in \Tb^+$
\end{itemize}
As we can see a diagnosis program has the same verification procedure with respect to both cautiously true $\Tc^+$ bravely false $\Tb^-$ test cases. The same holds for the cautiously false $\Tc^-$ and bravely true $\Tb^+$ test cases.
Therefore, in the following we can consider only the set of \emph{positive} test cases $P$ and the set of \emph{negative} test cases $N$ which are defined as:
\begin{align*}
& P := \Tc^+ \cup \setof{\overline{\tb^-} \;|\; \tb^- \in \Tb^-}  \\
 & N := \Tb^+ \cup \setof{\overline{\tc^-} \; |\; \tc^- \in \Tc^-}
\end{align*}

\begin{definition}\label{def:targetProg}
Let $\dpr[\pr,\mb]$ be a debugging program, $P$ be a set of positive test cases, $N$ be a set of negative test cases and $Er(\dpr[\pr,\mb])$ denote a set of all ground atoms over error-indicating predicates of $\dpr[\pr,\mb]$. A \emph{diagnosis problem} is to find such set of atoms $\md \subseteq Er(\dpr[\pr,\mb])$, called \emph{diagnosis}, such that the following requirements hold:

\begin{itemize}
	\item the diagnosis program $\dpr[\pr,\mb,\md]$ is  consistent,
	\item $\Int(\dpr[\pr,\mb,\md]) \models p \quad \forall p \in P$,
  \item $\Int(\dpr[\pr,\mb,\md]) \; \text{is consistent with} \; n \quad \forall n \in N$.
\end{itemize}
A tuple $\dpi$ is a \emph{diagnosis problem instance} (DPI).
\end{definition}
In the following we assume that the background theory $\mb$ together with the sets of test cases $P$ and $N$ always allow computation of the target diagnosis. That is, a user provides reasonable background knowledge as well as positive and negative test cases that do not interfere with each other.
\begin{proposition}
A diagnosis $\md$ for a DPI $\dpi$ does not exists if either (i) $\Delta' := \dpr[\pr,\mb] \cup \setof{a_i\leftarrow |\; a_i \in p, \forall p\in P}$ is inconsistent or (ii) $\exists n \in N$ such that the program $\Delta' \cup \setof{a_i \leftarrow |\; a_i \in n }$ is inconsistent.
\end{proposition}
\begin{proof} In the first case if $\Delta'$ is inconsistent, then either $\dpr[\pr,\mb]$ has no answer sets or every answer set of $\dpr[\pr,\mb]$ comprises an atom over $int/1$ or $\overline{int}/1$ predicate complimentary to some atom of a test case $p \in P$. The latter means that for any $\md \subseteq Er(\dpr[\pr,\mb])$ there exists $p\in P$ such that $\dpr[\pr,\mb,\md] \not\models p$. In the second case there exists a negative test case which is not consistent with any possible diagnosis program $\dpr[\pr,\mb,\md]$ for any $\md \subseteq Er(\dpr[\pr,\mb])$. Therefore in neither of the two cases requirements given in Definition~\ref{def:targetProg} can be fulfilled for any $\md \subseteq Er(\dpr[\pr,\mb])$. 
\end{proof}

Verification whether a set of atoms over error-indicating predicates is a diagnosis with respect to Definition~\ref{def:targetProg} can be done according to the following proposition.

\begin{proposition}
Let $\dpi$ be a DPI. Then, a set of atoms $\md \subseteq Er(\dpr[\pr,\mb])$ is a diagnosis for $\dpi$ iff 
$\dpr' := \dpr[\pr,\mb,\md]  \cup \bigcup_{p\in P} \setof{a_i\leftarrow |\; a_i \in p} $
 is consistent and
$\forall n \in N \;:\;  \dpr' \cup \setof{a_i\leftarrow |\; a_i \in n}$
is consistent.
\end{proposition}
\begin{proof} (sketch) ($\Rightarrow$) Let $\md$ be a diagnosis for $\dpi$. Since $\dpr[\pr,\mb, \md]$ is consistent and $\Int(\dpr[\pr,\mb, \md]) \models p$ for all $p\in P$ it follows that $\dpr[\pr,\mb, \md] \cup \bigcup_{p\in P} \setof{a_i\leftarrow |\; a_i \in p}$ is consistent. 
The latter program has answer sets because every $p\in P$ is a subset of every $\dI \in \Int(\dpr[\pr,\mb, \md])$. 
In addition, since the set of meta-interpretations $\Int(\dpr[\pr,\mb,\md]) \; \text{is consistent with every} \; n\in N$ there exists such set $\dI \in \Int(\dpr[\pr,\mb, \md])$ that $n \subseteq \dI$. Therefore the program $\dpr[\pr,\mb,\md] \cup \setof{a_i\leftarrow |\; a_i \in n}$ has at least one answer set. Taking into account that $\dpr'$ is consistent we can conclude that $\dpr' \cup \setof{a_i\leftarrow |\; a_i \in n}$ is consistent as well.

($\Leftarrow$) Let $\md \subseteq Er(\dpr[\pr,\mb])$ and $\dpi$ be a DPI. Since $\dpr'$ is consistent the diagnosis program $\dpr[\pr,\mb,\md]$ is also consistent. Moreover, for all $p\in P$ $\Int(\dpr[\pr,\mb,\md]) \models p$ because $\setof{a_i\leftarrow |\; a_i \in p} \subseteq \dpr'$. Finally, for every $n \in N$ consistency of $\dpr' \cup \setof{a_i\leftarrow |\; a_i \in n}$ implies that there must exist an interpretation $\dI \in \Int(\dpr[\pr,\mb,\md])$ satisfying $n$.
\end{proof}

\begin{definition}
A diagnosis $\md$ for a DPI $\dpi$ is a \emph{minimal diagnosis} iff  there is no diagnosis $\md^\prime$ such that $|\md^\prime|<|\md|$. 
\end{definition}
In our approach we consider only minimal diagnoses of a DPI since they might require less changes to the program than non-minimal ones and, thus, are usually preferred by users.
However, this does not mean that our debugging approach is limited to minimal diagnoses of an initial DPI. As we will show in the subsequent sections the interactive debugger acquires test cases and updates the DPI automatically such that all possible diagnoses of the initial DPI are investigated.
Computation of minimal diagnoses can be done by extension of the debugging program with such optimization criteria that only answer sets including minimal number of atoms over error-indicating predicates are returned by a solver.
Also, in practice a set of all minimal diagnoses is often approximated by a set of $n$ diagnoses in order to improve the response time of a debugging system. 
 
Computation of $n$ diagnoses for the debugging program $\dpr[\pr,\mb]$ of a problem instance $\dpi$ is done as shown in Algorithm~\ref{algo:diag}. The algorithm calls an ASP solver to compute one answer set $as$ of the debugging program (line 3). In case $\dpr[\pr,\mb]$ has an answer set the algorithm obtains a set $\md$ (line 5) and generates a diagnosis program $\dpr[\pr,\mb,\md]$ (line 6). The latter, together with the sets of positive and negative test cases is used to verify whether $\md$ is a diagnosis or not (line 7). All diagnoses are stored in the set $\mD$. In order to exclude the answer set $as$ from $AS(\dpr[\pr,\mb])$ the algorithm calls the \textsc{exclude} function (line 8) which extends the debugging program with the following integrity constraint, where atoms $d_1,\dots, d_n \in \md$ and $d_{n+1},\dots, d_m \in Er(\dpr[\pr,\mb])\setminus\md$:
\begin{equation*}
\leftarrow d_1, \dots, d_n, not\ d_{n+1}, \dots, not\ d_m 
\end{equation*}

\begin{algorithm*}[bt]
\caption{\textsc{ComputeDiagnoses$(\dpi, n)$}} \label{algo:diag}
\KwIn{DPI $\dpi$, maximum number of minimal diagnoses $n$}
\KwOut{a set of diagnoses $\mD$} 
\SetKwFunction{verify}{\normalfont \textsc{verify}}
\SetKwFunction{getAS}{\normalfont \textsc{getAnswerSet}}
\SetKwFunction{getDiagnosis}{\normalfont \textsc{getDiagnosis}}
\SetKwFunction{gen}{\normalfont \textsc{diagnosisProgram}}
\SetKwFunction{exclude}{\normalfont \textsc{exclude}}

$\mD \leftarrow \emptyset$\;
\While{$|\mD| < n $} {
    $as \leftarrow \getAS(\dpr[\pr,\mb])$\;
    \lIf{$as=\emptyset$} \textbf{exit loop}\;
    $\md \leftarrow as \cap Er(\dpr[\pr,\mb]))$\;    
    $\dpr[\pr,\mb,\md] \leftarrow \gen(\dpr[\pr,\mb],\md)$\;
    \lIf{$\verify(\dpr[\pr,\mb,\md], P, N)$} {$\mD \leftarrow \mD\cup\setof{\md}$}
    $\dpr[\pr,\mb] \leftarrow \exclude(\dpr[\pr,\mb], \md)$\;
} 

\Return $\mD$\;
\end{algorithm*}

Note, similarly to the model-based diagnosis~\cite{Reiter87,dekleer1987} our approach assumes that each error-indicating atom $er \in \md$ is relevant to an explanation of a fault, whereas all other atoms $Er(\dpr[\pr])\setminus \md$ are not. That is, some interpretations are not an answer sets of a program only because of reasons suggested by a diagnosis. Consequently, if a user selects a diagnosis $\md$ resulting in the debugging process, i.e.\ declares $\md$ as a correct explanation of a fault, then all other diagnoses automatically become incorrect explanations.

\vspace{5pt}\noindent\emph{Example} Let us exemplify our debugging approach on the following program $\pr_e$:
\begin{align*}
&\rl_1 : a \leftarrow not\ d \quad && \rl_2 : b \leftarrow a && \quad \rl_3 : c \leftarrow b \\ 
&\rl_4 : d \leftarrow c \quad && \rl_5:\leftarrow d &&   
\end{align*}
Assume also that the \emph{background theory} $\mb = \setof{\leftarrow d}$ and, therefore, the debugging program $\dpr[\pr_e, \mb]$ comprises two integrity constraints:
\begin{align*}
&\leftarrow \, rule(id_{r_5}),\, violated(id_{r_5}) \\
&\leftarrow \, rule(id_{r_5}),\ unsatisfied(id_{r_5})
\end{align*}

Since the program $\pr_e$ is inconsistent, a user runs the debugger to clarify the reason. In fact, the inconsistency is caused by an odd loop. That is, if $d$ is set to false, then the body of the rule $\rl_1$ is satisfied and $a$ is derived. However, given $a$ and the remaining rules $d$ must be set to true. In case when $d$ is true, $a$ is not derived and, consequently, there is no justification for $d$. 
The debugging program $\dpr[\pr_e, \mb]$ of a $\mathit{DPI}_1 := \tuple{\dpr[\pr_e, \mb], \emptyset, \emptyset}$ has $16$ answer sets. The addition of optimization criteria allows to reduce the number of answer sets to $4$ comprising only the minimal number of atoms over the error-indicating predicates. Since both sets of test cases are empty, a projection of these answer sets to the error-indicating predicates results in the following diagnoses:
\begin{align*}
\md_1:\setof{unsatisfied(id_{\rl_1})}\quad \md_2:\setof{unsatisfied(id_{\rl_2})} \\ \md_3:\setof{unsatisfied(id_{\rl_3})}\quad \md_4:\setof{unsatisfied(id_{\rl_4})}
\end{align*}

Definition~\ref{def:targetProg} allows to identify the target (preferred) diagnosis $\md_t$ for the program $\pr_e$ by providing sufficient information in the sets $\mb$, $P$ and $N$. 
Assume that $\mathit{DPI}_1$ is updated with two test cases -- one positive $\setof{int(a)}$ and one negative $\setof{\overline{int}(b)}$ -- and the debugger generates $\mathit{DPI}_2 := \tuple{\dpr[\pr_e, \mb], \setof{\setof{int(a)}}, \setof{\setof{\overline{int}(b)}}}$. These test cases require $\Int(\dpr[\pr_e,\mb,\md_t]) \models \setof{int(a)}$ and $\Int(\dpr[\pr_e,\mb,\md_t])$ to be consistent with $\setof{\overline{int}(b)}$ correspondingly.
Given this information the debugger will return only one diagnosis in our example, namely $\md_2$, since $\Int(\dpr[\pr_e,\mb,\md_2]) \models \setof{int(a)}$ and $\Int(\dpr[\pr_e,\mb,\md_2])$ is consistent with $\setof{\overline{int}(b)}$.
Indeed, a simple correction of $\pr_e$ by a user removing the rule $\rl_2$ results in a consistent program $\pr_2$ such that all new answer sets of $\pr_2$ fulfill all given test cases.
All other sets of atoms $\md_1, \md_3, \md_4$ are not diagnoses of $\mathit{DPI}_2$ because they violate the requirements. Thus, $\Int(\dpr[\pr_e,\mb,\md_1]) \not\models \setof{int(a)}$ and $\Int(\dpr[\pr_e,\mb,\md_i])$ is not consistent with $\setof{\overline{int}(b)}$ for $\md_i \in \setof{\md_3,\md_4}$. Consequently, $\md_2$ is the only possible diagnosis and it is accepted by a user as the target diagnosis $\md_t$.

\section{Query-based diagnosis discrimination}\label{sect:discrimination}

The debugging system might generate a set of diagnoses for a given DPI. In our example for simple $\mathit{DPI}_1$ the debugger returns four minimal diagnoses $\{\md_1, \dots, \md_4\}$. 
As it is shown in the previous section, additional information, provided in the background theory and test cases of a DPI $\dpi$ can be used by the debugging system to reduce the set of diagnoses. However, in a general case the user does not know which sets of test cases should be provided to the debugger s.t.\ the target diagnosis can be identified. That is, in many cases it might be difficult to provide a complete specification of a debugging query localizing a fault. 
Therefore, the debugging method should be able to find an appropriate set of atoms $A\subseteq At(\pr)$ on its own and only query the user or some other oracle, whether these atoms are cautiously/bravely true/false in the interpretations associated with the target diagnosis. To generate a query for a set of diagnoses $\mD=\{\md_1,\dots,\md_n\}$ the debugging system can use the diagnosis programs $\dpr[\pr,\mb,\md_i]$, where $\md_i \in \mD$. 

Since in many cases different diagnoses explain why different sets of interpretations of a program $\pr$ are not its answer sets, we can use discrepancies between the sets of interpretations to discriminate between the corresponding diagnoses.
In our example, for each diagnosis program $\dpr[\pr_e,\mb,\md_i]$ an ASP solver returns a set of answer sets encoding an interpretation which is not an answer set of $\pr_e$ and a diagnosis, see Table~\ref{tab:entailex1}.
Without any additional information the debugger cannot decide which of these atoms must be true in the missing answer sets of $\pr_e$. 
To get this information the debugging algorithm should be able to access some oracle which can answer a number of queries. 

\begin{table*}%
\centering
\begin{tabular}{cl} 
Diagnosis & Interpretations \\ \hline \\ [-1.5ex]
$\md_1$ : $\mathit{unsatisfied}(id_{r_1})$&$\setof{\setof{\overline{int}(a),\overline{int}(b),\overline{int}(c),\overline{int}(d)}}$ \\
$\md_2$ : $\mathit{unsatisfied}(id_{r_2})$&$\setof{\setof{{int}(a),\overline{int}(b),\overline{int}(c),\overline{int}(d)}}$ \\
$\md_3$ : $\mathit{unsatisfied}(id_{r_3})$&$\setof{\setof{{int}(a),{int}(b),\overline{int}(c),\overline{int}(d)}}$ \\
$\md_4$ : $\mathit{unsatisfied}(id_{r_4})$&$\setof{\setof{{int}(a),{int}(b),{int}(c),\overline{int}(d)}}$ \vspace{5pt}\\\hline
\end{tabular}
\caption{Interpretations $\Int(\dpr[\pr_e,\mb,\md_i])$ for each of the diagnoses $\mD=\{\md_1,\dots,\md_4\}$.}
\label{tab:entailex1}
\end{table*}

\begin{definition} Let $\dpi$ be a DPI, then a \emph{query} is set of ground atoms $\qry \subseteq At(\pr)$.
\end{definition}

Each answer of an oracle provides additional information which is used to update the actual DPI $\dpi$. Thus, if an oracle answers
\begin{itemize}
	\item \emph{cautiously true}, the set $\setof{int(a) \; | \;a \in \qry}$ is added to $P$;
  \item \emph{cautiously false}, the set $\setof{\overline{int}(a) \; | \;a \in \qry}$ is added to $N$;
  \item \emph{bravely true},  the set $\setof{ int(a) \; | \;a \in \qry}$ is added to $N$;
  \item \emph{bravely false}, the set $\setof{\overline{int}(a) \; | \;a \in \qry}$ is added to $P$.
\end{itemize}
The goal of asking a query is to obtain new information characterizing the target diagnosis. For instance, the debugger asks a user about classification of the set of atoms $\setof{c}$. 
If the answer is \emph{cautiously true}, the new $\mathit{DPI}_3=\tuple{\dpr[\pr_e,\mb],\setof{\setof{int(c)}}, \emptyset}$ has only one diagnosis $\md_4$ which is the target diagnosis w.r.t.\ a user answer. All other minimal sets of atoms over error-indicating predicates are not diagnoses because they do not fulfill the necessary requirements of Definition~\ref{def:targetProg}. If the answer is \emph{bravely false}, then the set $\setof{\overline{int}(c)}$ is added to $P$ and $\md_4$ is rejected. Consequently, we have to ask an oracle another question in order to discriminate between the remaining diagnoses. Since there are many subsets of $At(\pr)$ which can be queried, the debugger has to generate and ask only those queries which allow to discriminate between the diagnoses of the current DPI.

\begin{definition}\label{def:partition}
Each diagnosis $\md_i \in \mD$ for a DPI $\dpi$ can be assigned to one of the three sets $\dx$, $\dnx$ or $\dz$ depending on the query $\qry$ where:
\begin{itemize}
\item $\md_i \in \dx$ if it holds that:
\begin{equation*}
\Int(\dpr[\pr,\mb,\md_i]) \models \setof{int(a) \; | \;a \in \qry}
\end{equation*}
\item $\md_i \in \dnx$ if it holds that:
\begin{equation*}
\Int(\dpr[\pr,\mb,\md_i]) \models \setof{\overline{int}(a) \; | \;a \in \qry}
\end{equation*}
\item $\md_i \in \dz$ if  $\md_i \not\in \left(\dx \cup \dnx\right)$
\end{itemize}
A \emph{partition} of the set of diagnoses $\mD$ with respect to a query $\qry$ is denoted by a tuple $\tuple{\qry,\dxi{i},\dnxi{i},\dzi{i}}$.

\end{definition}

Given a DPI we say that the diagnoses in $\dx$ predict a positive answer (\emph{yes}) as a result of the query $\qry$, diagnoses in $\dnx$ predict a negative answer (\emph{no}), and diagnoses in $\dz$ do not make any predictions. 
Note, the answer \emph{yes} corresponds to classification of the query to the set of positive test cases $P$, whereas the answer \emph{no} is a result of a  classification of the query to the set of negative test cases $N$.
Therefore, without limiting the generality, in the following we consider only these two  answers.

The notion of a partition has an important property. Namely, each partition $\tuple{\qry,\dxi{i},\dnxi{i},\dzi{i}}$ indicates the changes in the set of diagnoses after the sets of test cases of an actual DPI are updated with respect to the answer of an oracle.
\begin{property}\label{prop:removeDiag}
Let $\mD$ be a set of diagnoses for a DPI $\dpi$, $\qry$ be a query, $\tuple{\qry,\dxi{i},\dnxi{i},\dzi{i}}$ be a partition of $\mD$ with respect to $\qry$ and $v \in \setof{\mathit{yes},\mathit{no}}$ be an answer of an oracle to a query $\qry$. 
\begin{itemize}
	\item if $v = \mathit{yes}$, then the set of diagnoses $\mD'$ for the updated DPI $\tuple{\dpr[\pr,\mb], P', N}$ does not comprise any elements of $\dnx$, i.e. $\mD' \cap \dnx = \emptyset$ and $(\dx\cup \dz) \subseteq \mD'$.
\item if $v = \mathit{no}$, then for set of diagnoses $\mD'$ of the updated DPI $\tuple{\dpr[\pr,\mb], P, N'}$ it holds that $\mD' \cap \dx = \emptyset$ and $(\dnx \cup \dz)\subseteq \mD'$.
\end{itemize}
\end{property}
Consequently, depending on the answer of an oracle to a query $\qry$ the set of diagnoses of an updated diagnosis problem instance comprises either $\dx \cup \dz$ or $\dnx \cup \dz$.

\begin{algorithm*}[bt]
\caption{\textsc{FindPartitions$(\dpi,\mD)$}} \label{algo:part}
\KwIn{DPI $\dpi$, a set of diagnoses $\mD$}
\KwOut{a set of partitions $\mathbf{PR}$} 
\SetKwFunction{common}{\normalfont \textsc{commonAtoms}}
\SetKwFunction{gen}{\normalfont \textsc{generatePartition}}

$\mathbf{PR} \leftarrow \emptyset$\;
\ForEach{$\dxi{i} \in \pset{\mD}$}
{
  $\ent_i \leftarrow \common(\dxi{i})$\;
  $\qry_i \leftarrow \setof{a \;|\; int(a) \in \ent_i}$\;
  \If{$\qry_i \neq \emptyset$}
  {
     $\tuple{\qry_i,\dxi{i},\dnxi{i},\dzi{i}} \leftarrow$\ 
     $ \gen(\qry_i,\mD,\dxi{i})$\;
    
   \lIf{$\dnxi{i} \neq\emptyset$}{ $\mathbf{PR} \leftarrow \mathbf{PR} \cup \{\tuple{\qry_i,\dxi{i},\dnxi{i},\dzi{i}}\}$}
  }
}

\Return $\mathbf{PR}$\;
\end{algorithm*}

In order to generate queries, we have to investigate for which sets $\dx, \dnx \subseteq \mD$ a query exists that can be used to differentiate between them. A straight forward approach to query generation is to generate and verify all possible subsets of $\mD$. This is feasible if we limit the number $n$ of minimal diagnoses to be considered during the query generation and selection. For instance, given $n=9$ the algorithm has to verify $512$ partitions in the worst case. In general, the number of diagnoses $n$ must be selected by a user depending on personal time requirements. The larger is the value of $n$ the more time is required to compute a query, but an answer to this query will provide more information to a debugger. 

Given a set of diagnoses $\mD$ for a DPI $\dpi$ Algorithm~\ref{algo:part} computes a set of partitions $\bf{PR}$ comprising all queries that can be used to discriminate between the diagnoses in $\mD$. For each element $\dxi{i}$ of the power set $\pset{\mD}$ the algorithm checks whether there is a set of atoms common to all interpretations of all diagnoses in $\dxi{i}$. The function \textsc{commonAtoms} (line 3) returns an intersection of all sets $\dI \in \Int(\dpr[\pr,\mb,\md_j])$ for all $\md_j \in \dxi{i}$. 
Given a non-empty query the function \textsc{generatePartition} (line 6) uses Definition~\ref{def:partition} to obtain a partition by classifying each diagnosis $\md_k \in \mD \setminus \dxi{i}$ into one of the sets $\dxi{i}$, $\dnxi{i}$ or $\dzi{i}$. Finally, all partitions allowing to discriminate between the diagnoses, i.e.\ comprising non-empty sets $\dxi{i}$ and $\dnxi{i}$, are added to the set $\mathbf{PR}$. 

\vspace{5pt}\noindent\emph{Example (cont.)} Reconsider the set of diagnoses $\mD=\setof{\md_1,\md_2,\md_3,\md_4}$ for the DPI $\tuple{\dpr[\pr_e,\setof{\leftarrow d}],\emptyset,\emptyset}$. The power set $\pset{\mD} = \setof{\{\md_1\},\setof{\md_2},\dots, \setof{\md_1,\md_2,\md_3,\md_4}}$ comprises 15 elements, assuming we omit the element corresponding to $\emptyset$ since it does not allow to compute a query. 
In each iteration an element of $\pset{\mD}$ is assigned to the set $\dxi{i}$. For instance, the algorithm assigned $\dxi{0}=\{\md_1,\md_2\}$. In this case the set $\qry_0$ is empty since the set $\ent_0=\setof{\overline{int}(b),\overline{int}(c),\overline{int}(d)}$  (see Table~\ref{tab:entailex1}). Therefore, the set $\{\md_1,\md_2\}$ is rejected and removed from $\pset{\mD}$. Assume that in the next iteration the algorithm selected $\dxi{1}=\{\md_2,\md_3\}$, for which the set of common atoms $\ent_1 = \setof{int(a),\overline{int}(c),\overline{int}(d)}$ and, thus, $\qry_1=\setof{a}$. 
The remaining diagnoses $\md_1$ and $\md_4$ are classified according to Definition~\ref{def:partition}. That is, the algorithm selects the first diagnosis $\md_1$ and verifies whether $\Int(\dpr[\pr,\mb,\md_1]) \models  \setof{int(a)}$. 
Given the negative answer, the algorithm checks if $\Int(\dpr[\pr,\mb,\md_1]) \models \setof{\overline{int}(a)}$. Since the condition is satisfied the diagnosis $\md_1$ is added to the set $\dnxi{1}$. 
The second diagnosis $\md_4$ is added to the set $\dxi{1}$ as it satisfies the first requirement $\Int(\dpr[\pr,\mb,\md_4]) \models \setof{int(a)}$. 
The resulting partition $\tuple{\{a\},\{\md_2,\md_3,\md_4\},\{\md_1\}, \emptyset}$ is added to the set $\bf{PR}$. 
In general, Algorithm~\ref{algo:part} returns a large number of possible partitions and the debugger has to select the best one. A random selection might not be a good strategy as it can overload an oracle with unnecessary questions (see~\cite{jws12} for an evaluation of a random strategy). Therefore, the debugger has to decide query of which partition should be asked first in order to minimize the total number of queries to be answered. Query selection is the central topic of 
\emph{active learning}~\cite{SettlesBook} which is an area of machine learning developing methods that are allowed to query an oracle for labels of unlabeled data instances. Most of the query selection measures used in active learning can be applied within our approach.
In this paper, we discuss two query selection strategies, namely, myopic and one step look-ahead.

Myopic query strategies determine the best query using only the set of partitions $\mathbf{PR}$. A popular ``Split-in-half'' strategy prefers those queries which allow to remove a half of the diagnoses from the set $\mD$, regardless of the answer of an oracle. That is, ``Split-in-half'' selects a partition $\tuple{\qry_i,\dxi{i},\dnxi{i},\dzi{i}}$ such that $|\dxi{i}|=|\dnxi{i}|$ and $\dzi{i}=\emptyset$. In our example, $\tuple{\setof{b},\setof{\md_3,\md_4}, \setof{\md_1,\md_2},\emptyset}$ is the preferred partition, since the set of all diagnoses of an updated DPI will comprise only two elements regardless of the answer of an oracle.

One step look-ahead strategies, such as prior entropy or information gain~\cite{SettlesBook}, allow to find the target diagnosis using less queries by incorporating heuristics assessing the prior probability $p(\md_i)$ of each diagnosis $\md_i \in \mD$ to be the target one ~\cite{dekleer1987,jws12}. 
Such heuristics can express different preferences/expectations of a user for a fault explanation. For instance, one heuristic can state that rules including many literals are more likely to be faulty. Another heuristics can assign higher probabilities to diagnoses comprising atoms over $\mathit{unsatisfiable}/1$ predicate if a user expects this type of error. 
In addition, personalized heuristics can be learned by analyzing the debugging actions of a user in, e.g., ASPIDE~\cite{febbraro2011aspide} or SeaLion~\cite{SeaLion2011}. 

A widely used one step look-ahead strategy~\cite{dekleer1987} suggests that the best query is the one which, given the answer of an oracle, minimizes the expected entropy of the set of diagnoses. 
Let $p(\qry_i = v)$ denote the probability that an oracle gives an answer $v \in \setof{yes,no}$ to a query $\qry_i$ 
and $p(\md_j | \qry_i = v)$ be the probability of diagnosis $\md_j$ given an oracle's answer. 
The expected entropy after querying $\qry_i$ is computed as (see~\cite{jws12} for details):

\begin{align*}
H_e(\qry_i) =  & \sum_{v \in \setof{yes,no}}  p(\qry_i = v) \times \\
   &-\sum_{\md_j \in \mD} p(\md_j | \qry_i = v) \log_2 p(\md_j | \qry_i = v)
\end{align*}

The required probabilities can be computed from the partition $\tuple{\qry_i,\dxi{i},\dnxi{i},\dzi{i}}$ for the query $\qry_i$ as follows:
\begin{align*}
&p (\qry_i=yes)= p(\dxi{i}) + p(\dzi{i})/2 \\
&p(\qry_i=no)= p(\dnxi{i}) + p(\dzi{i})/2
\end{align*}
where the total probability of a set of diagnoses can be determined as: $p({\bf S_{i}}) = \sum_{\md_j \in {\bf S_{i}}}{p(\md_j)}$, since all diagnoses are considered as mutually exclusive, i.e.\ they cannot occur at the same time. The latter follows from the fact that the goal of the interactive debugging process is identification of \emph{exactly one} diagnosis that explains a fault and is accepted by a user. As soon as the user accepts the preferred diagnosis all other diagnoses become irrelevant. 
The total probability of diagnoses in the set $\dzi{i}$ is split between positive and negative answers since these diagnoses make no prediction about outcome of a query, i.e.\ both outcomes are equally probable. Formally, the probability of an answer $v$ for a query $Q_i$ given a diagnosis $\md_j$ is defined as:
\begin{equation*}
p(\qry_i=v|\md_j) = 
\begin{cases}
1, & \mbox{if $\md_j$ predicted $\qry_s=v$;}\\
0, & \mbox{if $\md_j$ is rejected by $\qry_s=v$;}\\
\frac{1}{2}, & \mbox{if $\md_j \in \dzi{s}$}.
\end{cases} 
\end{equation*}
The probability of a diagnosis given an answer, required for the calculation of the entropy, can be found using the Bayes rule:
\begin{equation*}
p(\md_j|\qry_i=v) = \frac{p(\qry_i=v|\md_j)p(\md_j)}{p(\qry_i = v)}
\label{eq:dupdate}
\end{equation*}

After a query $\qry_s$ is selected by a strategy 
\[ \qry_s = \argmin_{Q_i} H_e(Q_i)\]
the system asks an oracle to provide its classification. Given the answer $v$ of an oracle, i.e.\ $\qry_s = v$, we have to update the probabilities of the diagnoses to take the new information into account. The update is performed by the Bayes rule given above.

\begin{algorithm*}[bt]

\caption{\textsc{InteractiveDebugging$(\pr, S, \mb, P, N, H, n, \sigma)$}} \label{algo:general}
\KwIn{ground disjunctive program $\pr$, query selection strategy $S$, background knowledge $\mb$, sets of positive $P$ and negative $N$ test cases, set of  heuristics $H$, maximum number minimal diagnoses $n$, acceptance threshold $\sigma$}
\KwOut{a diagnosis $\md$} 
\SetKwFunction{HSTree}{\normalfont \textsc{ComputeDiagnoses}}
\SetKwFunction{computeDataSet}{\normalfont \textsc{computeDataSet}}
\SetKwFunction{getSize}{\normalfont \textsc{getSize}}
\SetKwFunction{uptateProbablities}{\normalfont \textsc{uptateProbablities}}
\SetKwFunction{computePriors}{\normalfont \textsc{computePriors}}
\SetKwFunction{getAnswer}{\normalfont \textsc{getAnswer}}
\SetKwFunction{getMinimalEntopyScore}{\normalfont \textsc{getMinimalScore}}
\SetKwFunction{selectMeasurement}{\normalfont \textsc{selectQuery}}
\SetKwFunction{getDiag}{\normalfont \textsc{mostProbableDiagnosis}}
\SetKwFunction{thresh}{\normalfont \textsc{aboveThreshold}}
\SetKwFunction{exit}{exit loop}
\SetKwFunction{getScore}{\normalfont \textsc{getScore}}
\SetKwFunction{getQuery}{\normalfont \textsc{getQuery}}
\SetKwFunction{isempty}{\normalfont \textsc{isQueryEmpty}}
\SetKwFunction{gen}{\normalfont \textsc{FindPartitions}}
\SetKwFunction{dpigen}{\normalfont \textsc{generateDPI}}
\SetKwFunction{dpiupd}{\normalfont \textsc{updateDPI}}
\SetKwFunction{diagupd}{\normalfont \textsc{updateDiagnoses}}

$\dpi \leftarrow \dpigen(\pr,\mb)$;
$\mD \leftarrow \emptyset$\;

\Repeat{{$\thresh(\mD, H, \sigma) \lor |\mD| \leq 1$}} {
    $\mD \leftarrow \mD \cup \HSTree (\dpi, n-|\mD|)$\;    
    $\mathbf{PR} \leftarrow \gen(\dpi, \mD)$\;
    $Q \leftarrow \selectMeasurement(\mathbf{PR}, H, S)$\;
    \lIf {$Q = \emptyset$} \textbf{exit loop}\;
    $A \leftarrow \getAnswer(\qry)$\;
    $\dpi \leftarrow \dpiupd(A, \dpi)$\;
    $\mD \leftarrow \diagupd(A, \qry, \mathbf{PR}, H)$\;
} 

\Return $\getDiag(\mD, S, H)$\;
\end{algorithm*}

In order to reduce the number of queries a user can specify a threshold, e.g. $\sigma=0.95$. If the absolute difference in probabilities between two most probable diagnoses is greater than this threshold, the query process stops and returns the most probable diagnosis. 

Note that, in the worst case the number of queries required to find the preferred diagnosis equals to the number of diagnoses of the initial DPI. In real-world applications, however, the worst case scenario is rarely the case. It is only possible if a debugger always prefers queries of such partitions $\tuple{\qry_i,\dxi{i},\dnxi{i},\dzi{i}}$ that either $|\dxi{i}|=1$ or $|\dnxi{i}|=1$ and an answer of an oracle always unfavorable. That is, only one diagnosis of the actual DPI will not appear the set of diagnoses of the updated DPI. 

We have not found any representative set of faulty ASP programs for which the preferred explanation of a fault, i.e.\ the target diagnosis, is known. Therefore, we do not report in this paper about the number of queries required to find such diagnosis. However, the evaluation results presented in~\cite{jws12} show that only a small number of queries is usually required to find the preferred diagnosis. In the worst case their approach asked 12 queries on average to find the preferred diagnosis from over 1700 possible diagnoses. In better cases only 6 queries were required.  This study indicates a great potential of the suggested method for debugging of ASP programs. We plan verify this conjecture in out future work. In addition, our approach can use RIO~\cite{Rodler2013}, which is a query strategy balancing method that automatically selects the best query selection strategy during the diagnosis session, thus, preventing the worst case scenario.

The interactive debugging system (Algorithm~\ref{algo:general}) takes a ground program or a ground instantiation of non-ground program as well as a query selection strategy as an input. Optionally a user can provide background knowledge, relevant test cases as well as a set of heuristics assessing probabilities of diagnoses. If the first three sets are not specified, then the corresponding arguments are initialized with $\emptyset$. In case a user specified no heuristics, we add a simple function that assigns a small probability value to every diagnosis. 
The algorithm starts with the initialization of a DPI. The debugging program $\dpr[\pr,\mb]$ is generated by \texttt{spock}\footnote{www.kr.tuwien.ac.at/research/debug}, which implements the meta-programming approach of~\cite{Gebser2008b}. First, the main loop of Algorithm~\ref{algo:general} computes the required number of diagnoses such that $|\mD|=n$. Next, we find a set of partitions for the given diagnoses and select a query according to a query strategy $S$ selected by a user. If the user selected the myopic strategy then probabilities of diagnoses are ignored by \textsc{selectQuery}. The oracle is asked to classify the query and the answer is used to update the DPI as well as a the set $\mD$ from which we remove all elements that are not diagnoses of the updated DPI.
The main loop of the algorithm exits if either there is a diagnosis which probability satisfies the threshold $\sigma$ or only one diagnosis remains. Finally, the most probable diagnosis or, in case of a myopic strategy, the first diagnosis is returned to a user. Algorithm~\ref{algo:general} was prototypically implemented as a part of a general diagnosis framework\footnote{https://code.google.com/p/rmbd/wiki/AspDebugging}. A plug-in for SeaLion providing a user-friendly interface for our interactive debugging method is currently in development.

\section{Summary and future work}

In this paper we presented an approach to the interactive query-based debugging of disjunctive logic programs. 
The differentiation between the diagnoses is done by means of queries which are automatically generated from answer sets of the debugging meta-program. Each query partitions a set of diagnoses into subsets that make different predictions for an answer of an oracle. 
Depending on the availability of heuristics assessing the probability of a diagnosis to be the target one, the debugger can use different query selection strategies to find the most informative query allowing efficient identification of the target diagnosis.

In the future work we are going to investigate the applicability of our approach to the method of~\cite{Oetsch2010a} since (a) this method can be applied to non-grounded programs and (b) it was recently extended to programs with choice rules, cardinality and weight constraints~\cite{PolleresFSF13}.
%
In addition, there is a number of other debugging methods for ASP that might be integrated with the suggested query selection approach. For instance, the method of \cite{Mikitiuk2007} can be used to translate the program and queries into a natural language representation, thus, simplifying the query classification problem. Another technique that can be used to simplify the query answering is presented in \cite{Pontelli2009} where the authors suggest a graph-based justification technique for truth values with respect to an answer set. 
Moreover, we would like to research whether query generation and selection ideas can be applied in the debugging method of~\cite{Oetsch2011step}. This interactive framework allows a programmer to step through an answer set program by iteratively extending a state of a program (partial reduct) with new rules. The authors suggest a filtering approach that helps a user to find such rules and variable assignments that can be added to a state. We want to verify whether the filtering can be extended by querying about disagreements between the  next states, such as ``if user adds a rule $r_1$ then $r_2$ cannot be added''.

One more interesting source of heuristics, that we also going to investigate, can be obtained during testing of ASP programs~\cite{Janhunen2010}. The idea comes from spectrum-based fault localization (SFL)~\cite{Harrold1998}, which is widely applied to software debugging. Given a set of test cases specifying inputs and outputs of a program SFL generates an observation matrix $A$ which comprises information about: (i) parts of a program executed for a test case and (ii) an error vector $E$ comprising results of tests executions. Formally, given a program with $n$ software components  $C:=\setof{c_1,\dots,c_n}$ and a set of test cases $T:=\setof{t_1,\dots,t_m}$ a hit spectra is a pair $(A,E)$. $A$ is a $n\times m$ matrix where each $a_{ij} = 1$ if $c_j$ was involved in execution of the test case $t_i$ and  $a_{ij} = 0$ otherwise. Similarly for each $e_i \in E$, $e_i=1$ if the test case $t_i$ failed and $e_i=0$ in case of a success. Obviously, statistics collected by the hit spectra after execution of all tests allows to determine the components that were involved in execution of failed test cases. Consequently, we can obtain a set of fault probabilities for the components $C$. The same methodology can be applied to debugging and testing of ASP programs. For each test case $t_i$ we have to keep a record which sets of ground rules (Gelfond-Lifschitz reducts) were used to obtain answer sets that violate/satisfy $t_i$. Next, we can use the obtained statistics to derive fault probabilities for ground rules of an ASP program being debugged. The probabilities of diagnoses can then be computed from the probabilities of rules as it is shown in~\cite{jws12}.

\section{Acknowledgments}
The authors would like to thank Gerhard Friedrich and Patrick Rodler for the discussions regarding query selection strategies. We are also very thankful to anonymous reviewers for their helpful comments.


\end{document}